\documentclass{article} 
\usepackage{iclr2027_conference_arxiv,times}

\usepackage{amsmath,amsfonts,bm}

\def\eqref#1{equation~\ref{#1}}

\def\1{\bm{1}}

\DeclareMathAlphabet{\mathsfit}{\encodingdefault}{\sfdefault}{m}{sl}
\SetMathAlphabet{\mathsfit}{bold}{\encodingdefault}{\sfdefault}{bx}{n}

\newcommand{\E}{\mathbb{E}}

\newcommand{\R}{\mathbb{R}}

\newcommand{\Var}{\mathrm{Var}}

\usepackage{hyperref}
\usepackage{url}

\usepackage{booktabs}       
\usepackage{amsfonts}       
\usepackage{nicefrac}       
\usepackage{microtype}      
\usepackage{xcolor}         
\usepackage{amsmath,amssymb,amsthm,mathtools}
\usepackage{algorithm}
\usepackage{algorithmic}
\usepackage{booktabs}
\usepackage{graphicx}
\usepackage{subcaption}
\usepackage{microtype}
\usepackage{enumitem}

\newtheorem{proposition}{Proposition}

\newtheorem{theorem}{Theorem}
\newtheorem{corollary}{Corollary}

\newcommand{\relu}{\operatorname{ReLU}}

\newcommand{\inc}{\mathrm{inc}}
\newcommand{\ADIP}{\textsc{ADIP}}
\newcommand{\SCIP}{\textsc{SCIP}}

\title{Autoregressive Differentiable Method for Integer Programming}

\author{Ouns El Harzli\\
BCG X AI Science Institute\\
Boston, USA \\
\texttt{elharzli.ouns@bcg.com} \\
\And
Yudong Cao \\
BCG X AI Science Institute \& Zapata Quantum\\
Boston, USA\\
\texttt{david.yudong.cao@gmail.com} \\
}

\iclrfinalcopy 
\begin{document}

\maketitle

\begin{abstract}
We introduce an autoregressive differentiable method to solve 0-1 integer programs. We fix an arbitrary order of the binary variables and we train a transformer to predict the next bit while remaining in the feasible set. Our method is first trained on feasible incumbents provided by any solver, thus allowing us to initialize the transformer in the feasible set. Our procedure then implements a Lagrangian penalty to penalize infeasible solutions, and the transformer is further trained to explore the feasible set using Gumbel-softmax activations on the relaxed objective. We have tested our method on non-convex instances of quadratic knapsack problem and demonstrated consistent improvement upon state-of-the-art open-source solvers for dense problems up to 10,000 binary variables. In particular, we empirically demonstrate a phenomenon akin to a tunneling effect where the effective change of variables from binary variable to the continuous weights of the transformer that the method implements enables crossing barriers in the relaxed objective landscape.
\end{abstract}

\section{Introduction}

Binary integer programming is a central modeling language for combinatorial optimization \citep{nemhauser1988integer}. Problems including knapsack \citep{martello1990knapsack},
maximum cut \citep{goemans1995improved},
facility location \citep{cornuejols1990uncapacitated},
sparse regression \citep{bertsimas2016bestsubset},
and scheduling \citep{pinedo2016scheduling}
can be formulated over binary vectors, yet many practically relevant instances remain computationally challenging \citep{belotti2013mixed,bertsimas2016bestsubset,bixby2012brief,lodi2010mixed,gleixner2021miplib}. Modern branch-and-cut solvers and primal heuristics are extremely powerful, but can still spend substantial time improving a good feasible incumbent \citep{berthold2013measuring,fischetti2003local,rothberg2007evolutionary}.

We ask whether autoregressive differentiable models can serve as enhanced primal mechanisms for binary integer programs \citep{vinyals2015pointer,bello2017neural,geng2025differentiable}. Rather than predicting a solution from scratch, we initialize a neural model at a feasible incumbent obtained from a standard solver \citep{nair2020solving,khalil2022mipgnn,han2023gnn}, then optimize it directly against a differentiable relaxation of the original objective with a Lagrangian penalty for constraint violation \citep{dipillo1989exact}. The resulting method searches in neural network weight space rather than directly in the decision variables.
This formulation is intended to complement conventional integer programming machinery: the solver supplies a feasible point and the neural reparameterization is used only for subsequent primal improvement. The central question is whether the nonlinear, shared autoregressive parametrization exposes useful optimization directions that differ from direct coordinate-wise movement in the binary decision space.

Our starting point is deliberately instance-specific. The neural network is not trained to generalize across a dataset of optimization instances; instead, it is initialized at a solution already found by a solver and used as a local improvement mechanism. This separates feasibility discovery from neural improvement and lets the transformer exploit a high-quality incumbent without learning the combinatorial problem from scratch.

We focus on binary quadratic integer programs
\begin{equation}
    \max_{x\in\{0,1\}^n} x^\top Qx \qquad \text{s.t.}\qquad Ax\le b,
    \label{eq:bqip}
\end{equation}
where $Q\in\R^{n\times n}$, $A\in\R^{m\times n}$, and $b\in\R^m$. Our experimental testbed is the quadratic knapsack problem (QKP), following and extending \citet{lopezpiqueres2025constraining}. QKP is known to remain strongly NP-hard even with small integer weights \citep{gallo1980quadratic,pisinger2007solution}, and arbitrary $Q$ yields a non-convex relaxed landscape. Although we experiment with non-convex quadratic objectives, the method applies whenever the binary objective admits a differentiable continuous relaxation.

Our method, which we call Autoregressive Differentiable Integer Progamming (\ADIP{}), has two phases. We fix an arbitrary variable ordering and use a causal transformer to learn
$p_\theta(x)=\prod_{i=1}^n p_\theta(x_i\mid x_{<i})$. For each instance, an open-source solver such as \SCIP{} \citep{hojny2025scip10}, HiGHS \citep{huangfu2018highs}, CP-SAT \citep{perron2023ortools}, or CBC \citep{forrest2005cbc,cbc2026} first produces a feasible incumbent $x^{\inc}$. In Phase~1, the transformer is trained to reproduce $x^{\inc}$ using targets $\varepsilon$ and $1-\varepsilon$ (instead of $0$ and $1$), preventing softmax saturation and preserving gradients. In Phase~2, the same weights are optimized against
\begin{equation}
\mathcal{J}_\lambda(\theta)=\tilde{x}_\theta^\top Q\tilde{x}_\theta-\lambda\sum_{j=1}^m\relu((A\tilde{x}_\theta-b)_j),
\label{eq:lagrangian_objective_intro}
\end{equation}
where $\tilde{x}_\theta\in[0,1]^n$ is generated autoregressively. Gumbel-softmax activations \citep{jang2017categorical,maddison2017concrete} encourage near-discrete exploration, while the best feasible hard solution encountered is retained throughout training.

For the purpose of advancing open-source software and open science, we deliberately chose to compare against state-of-the-art open-source solvers, and not against their closed-source competitors. Our contributions are fourfold. First, we introduce an autoregressive nonlinear reparametrization of relaxed binary decision variables for integer programs, in which the shared parameters of the autoregressive model themselves become the optimization variables. Second, we present our specific method \ADIP{} combining transformer-based incumbent imitation, finite-logit targets, Gumbel-softmax exploration, and a Lagrangian feasibility penalty in a two-phase primal-improvement procedure. Third, we provide theoretical results showing that incumbent improvement remains NP-hard, that the transformer reparametrization retains the binary search space, and that parameter-space gradient flow induces a state-dependent geometry in decision space. Fourth, we evaluate \ADIP{} on random QKP instances under matched computational budgets and study whether its trajectories reach attraction basins different from those reached by combinatorial optimization or by direct relaxed optimization, uncovering a ``tunneling-like'' phenomenon.

\begin{figure*}[t]
    \centering
    \includegraphics[width=\textwidth]{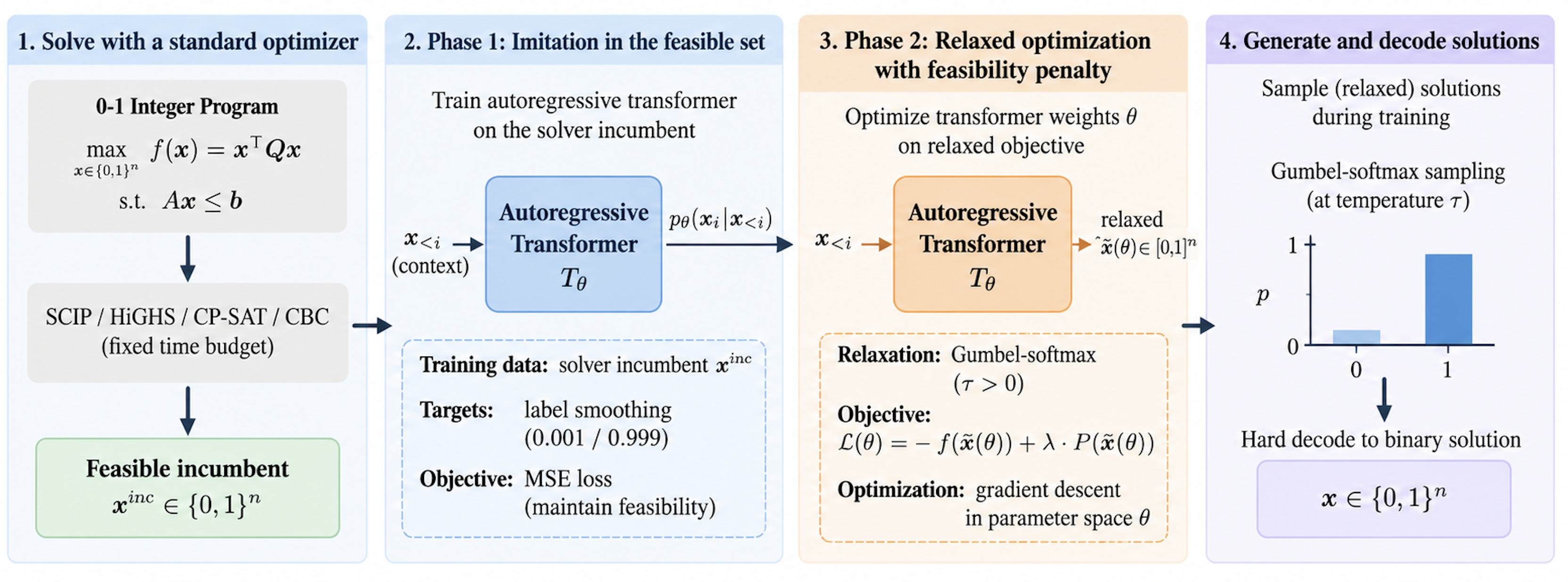}
    \caption{
    \textbf{Overview of ADIP.}
    Starting from a feasible incumbent $x^{\mathrm{inc}}$ obtained from a standard
    integer-programming solver, Phase~1 trains an autoregressive transformer to
    reproduce the incumbent using smoothed binary targets.
    Phase~2 optimizes the transformer parameters directly against the relaxed
    integer-program objective with a Lagrangian feasibility penalty, using
    Gumbel-softmax to maintain near-discrete outputs.
    Hard candidates are decoded throughout training, and the best feasible
    solution encountered is returned.
    }
    \label{fig:adip_overview}
\end{figure*}

\section{Related work}
Our method is closely related in motivation to Differentiable Integer Linear Programming (DiffILO) \citep{geng2025differentiable}, which constructs a probabilistic continuous reformulation of an ILP. The two approaches differ in how differentiability is obtained and in the role of the neural parametrization. DiffILO associates independent Bernoulli parameters with binary decisions and
derives a stochastic constraint surrogate. \citet{sornwanee2026diffilo} shows that its
sample-wise loss contains a discontinuous indicator, so the differentiation-expectation interchange used to justify its stochastic gradient is generally
invalid and omits boundary terms that are not scalable for general constraints.
ADIP instead differentiates an explicit continuous Gumbel-softmax
autoregressive relaxation \citep{jang2017categorical,maddison2017concrete}.

Autoregressive models have previously been used for combinatorial optimization and mixed-integer programming. \citet{tang2025autoregressive} use autoregressive reinforcement learning to sequentially fix MILP variables before a conventional solver completes the solution, while variational neural annealing uses autoregressive distributions for Ising-type combinatorial optimization \citep{hibatallah2021variational}. In contrast, \ADIP{} first fits an autoregressive model to a feasible solver incumbent and then makes its shared parameters the optimization variables themselves: the original relaxed objective is backpropagated through the nonlinear map $\theta\mapsto\tilde{x}_\theta$. To the best of our knowledge, this combination of incumbent initialization, autoregressive reparametrization, and direct differentiable optimization of the integer-program objective has not previously been proposed for general binary integer programming. 

ADIP is also related to quantum-inspired generative optimization.
Constrained tensor networks \citep{lopezpiqueres2025constraining} encode
linear constraints exactly into matrix product states and optimize a distribution supported
only on feasible configurations; their QKP benchmark directly motivates ours.
In contrast, ADIP does not encode the feasible set architecturally, but
initializes from a feasible incumbent and optimizes through a Lagrangian
penalty. The Generative Quantum Eigensolver
\citep{nakaji2024generative} likewise places the trainable parameters in a
classical autoregressive model, but uses the transformer to generate quantum
circuits rather than the decision variables themselves. Quantum Hamiltonian
Descent \citep{leng2023quantum} provides a complementary quantum mechanism
for traversing non-convex landscapes via Hamiltonian dynamics and motivates
our tunneling analogy.

\section{Preliminaries}
We consider Eq.~\ref{eq:bqip} with feasible set $\mathcal F=\{x\in\{0,1\}^n:Ax\le b\}$ and continuous relaxation $\tilde{x}\in[0,1]^n$. Under a fixed variable ordering, a causal transformer $T_\theta$ \citep{vaswani2017attention} factorizes
\begin{equation}
p_\theta(x)=\prod_{i=1}^n p_\theta(x_i\mid x_{<i}).
\label{eq:autoregressive_factorization}
\end{equation}
We use a causal transformer with alphabet $\{"0", "1"\}$, and for differentiable self-generation, a relaxed bit is embedded via token embeddings $e_0, e_1$ as
\begin{equation}
e(\tilde{x}_i)=(1-\tilde{x}_i)e_0+\tilde{x}_i e_1,
\label{eq:relaxed_embedding}
\end{equation}
which allows gradients to propagate through the generated prefix $(\tilde{x}_{\theta})_{<i}$.

For logits $\ell_{i,0}, \ell_{i,1} \in\R$, independent Gumbel noises $g_{i,0},g_{i,1}$, and temperature $\tau>0$, Gumbel--softmax \citep{jang2017categorical} gives
\begin{equation}
\pi_{i,k}=\frac{\exp((\ell_{i,k}+g_{i,k})/\tau)}{\sum_{r\in\{0,1\}}\exp((\ell_{i,r}+g_{i,r})/\tau)},\qquad \tilde{x}_i=\pi_{i,1}.
\label{eq:gumbel_softmax}
\end{equation}
In our experiments we use the straight-through variant: the forward pass uses $\hat{x}_i=\mathbf 1\{\pi_{i,1}\ge\pi_{i,0}\}$, while the backward pass differentiates the continuous relaxation in Eq.~\ref{eq:gumbel_softmax}.

\section{Methodology}

We now describe \ADIP{}, the proposed autoregressive differentiable method for binary integer programming. 
The method has two phases. 
The first phase initializes a transformer at a feasible incumbent found by an external solver. 
The second phase optimizes the transformer weights using a differentiable relaxation of the objective with a feasibility penalty.

\subsection{Overview}
Given an instance of Eq.~\ref{eq:bqip}, \ADIP{} first runs a conventional solver for $T_{\mathrm{seed}}$ and extracts its best feasible incumbent $x^{\inc}$. A causal transformer is then trained under teacher forcing to reproduce this incumbent. Starting from those trained weights, Phase~2 generates a relaxed sequence autoregressively with Gumbel--softmax and ascends the penalized objective, while every hard-decoded candidate is checked against the original constraints and the best feasible solution is retained. Algorithm~\ref{alg:adip} summarizes the procedure.
\begin{algorithm}[t]
\caption{\ADIP{}: Autoregressive Differentiable Integer Programming}
\label{alg:adip}
\small
\begin{algorithmic}[1]
\REQUIRE $(Q,A,b)$, solver $\mathcal S$, $T_{\mathrm{seed}}$, $K_1,K_2$, penalty $\lambda$, temperature $\tau$
\STATE Run $\mathcal S$ for $T_{\mathrm{seed}}$; let $x^{\inc}$ be its best feasible incumbent
\STATE Initialize transformer parameters $\theta$
\FOR{$k=1,\ldots,K_1$}
\STATE Update $\theta$ to minimize $\mathcal L_{\mathrm{init}}$ in Eq.~\ref{eq:init_loss} under teacher forcing
\ENDFOR
\STATE $x^\star\leftarrow x^{\inc}$
\FOR{$k=1,\ldots,K_2$}
\STATE Generate $\tilde{x}^{(k)}_\theta$ autoregressively using Gumbel--softmax temperature $\tau$
\STATE Ascend $\nabla_\theta\mathcal J_\lambda(\theta)$ from Eq.~\ref{eq:phase2_objective}
\STATE Decode $\hat{x}^{(k)}$ and update $x^\star$ if it is feasible and improves the objective
\ENDFOR
\RETURN $x^\star$
\end{algorithmic}
\end{algorithm}

\subsection{Complexity of incumbent improvement}
It is worth noting that starting from a feasible incumbent does not make the residual search problem easy.
\begin{theorem}[Hardness of incumbent improvement]\label{thm:incumbent-hardness}
For QKP with integer-valued objective coefficients and strictly positive integer weights, given a feasible incumbent $x^{\inc}$, deciding whether there exists a feasible $x$ with $x^\top Qx>(x^{\inc})^\top Qx^{\inc}$ is NP-complete.
\end{theorem}
The proof, in Appendix \ref{incumbent-hardness}, reduces QKP decision to incumbent improvement by adding a capacity-saturating item of value $K-1$. Thus \ADIP{} addresses an NP-hard residual search problem even after a feasible point is provided.

\subsection{Expressivity of the autoregressive parametrization}
Universal approximation results for transformers \citep{yun2020transformers} imply that the neural change of variables does not discard binary solutions.
\begin{proposition}[Combinatorial expressivity]\label{prop:sequence-surjectivity}
Fix sequence length $n$ and a transformer class satisfying the universal approximation property of \citet{yun2020transformers}. For every $x^\star\in\{0,1\}^n$ and $\varepsilon\in(0,1/2)$, there exist parameters $\theta$ such that $\|\tilde{x}_\theta-x^\star\|_\infty<\varepsilon$. Hence threshold decoding at $1/2$ is surjective onto $\{0,1\}^n$.
\end{proposition}
\begin{corollary}[Hardness under autoregressive reparametrization]\label{cor:theta-hardness}
Under Proposition~\ref{prop:sequence-surjectivity} and given $M \in \R$, deciding whether there exists $\theta$ such that $\hat{x}_\theta^\top Q \hat{x}_\theta > M$ subject to $A\hat{x}_\theta\le b$ \footnote{This is the decision version of globally optimizing the objective via the hard-decoded transformer.} is NP-hard.
\end{corollary}
Proofs are in Appendices \ref{sequence-surjectivity}-\ref{theta-hardness}. The parametrization therefore retains the combinatorial search space and its worst-case hardness. This separates expressivity from optimization: the role of the transformer reparameterization is to provide a different continuous geometry through which the discrete search space is explored (see Paragraph \ref{sec:induced_geometry}).

\subsection{Phase 1: feasible incumbent imitation}
Let $x^{\inc}$ be the best feasible solver incumbent. For each position $i$, the transformer receives $x^{\inc}_{<i}$ and predicts $x_i^{\inc}$. Exact targets in $\{0,1\}$ can drive the logits toward saturation, so we use
\begin{gather}
y_i^{\inc}
=
\begin{cases}
1-\varepsilon, & x_i^{\inc}=1,\\
\varepsilon,   & x_i^{\inc}=0,
\end{cases}
\\[0.5em]
\mathcal{L}_{\mathrm{init}}(\theta)
=
-\frac{1}{n}
\sum_{i=1}^{n}
\left[
y_i^{\inc}
\log p_\theta(x_i=1\mid x^{\inc}_{<i})
+
(1-y_i^{\inc})
\log\!\left(
1-p_\theta(x_i=1\mid x^{\inc}_{<i})
\right)
\right],
\label{eq:init_loss}
\end{gather}
with $\varepsilon=0.001$ in our experiments. Here target smoothing enforces a finite logit gap $\log((1-\varepsilon)/\varepsilon)$, whereas exact targets are reached only in the infinite-logit limit. Thus Phase~1 yields a model whose hard decoding reproduces the feasible incumbent while retaining usable gradients for Phase~2.
 The decoded hard sequence agrees with $x^{\inc}$ whenever the predicted probabilities are on the correct side of $1/2$. At the same time, the relaxed output remains a finite-logit representation near the incumbent, so Phase~2 begins from a feasible discrete anchor without having destroyed the gradient signal needed to move away from it.

\subsection{Phase 2: differentiable optimization in transformer weight space}
Let $\tilde{x}_\theta\in[0,1]^n$ denote the relaxed autoregressive output and define
\begin{equation}
\Phi(\tilde{x})=\sum_{j=1}^m\relu((A\tilde{x}-b)_j),\qquad
\mathcal J_\lambda(\theta)=\tilde{x}_\theta^\top Q\tilde{x}_\theta-\lambda\Phi(\tilde{x}_\theta).
\label{eq:phase2_objective}
\end{equation}
We maximize $\mathcal J_\lambda$ with respect to $\theta$. During Phase~2 we monitor hard decoded feasibility $A\hat{x}_\theta\le b$; the returned solution is always the best feasible hard candidate encountered.

For integer-valued constraints, the incumbent admits a conservative finite penalty scale for binary solutions.
\begin{proposition}[Sufficient penalty for binary feasibility]\label{prop:lambda-bound}
Let $f_{\inc}=(x^{\inc})^\top Qx^{\inc}$ and $U_Q=\sum_{i,j}\max\{Q_{ij},0\}$. If $A,b$ are integer-valued and $\lambda>U_Q-f_{\inc}$, every infeasible binary $x$ satisfies $x^\top Qx-\lambda\Phi(x)<f_{\inc}$; hence no global maximizer of the penalized binary objective is infeasible.
\end{proposition}
 The proof is in Appendix \ref{lambda-bound}. The guarantee does not extend to arbitrary relaxed points, whose violations can be arbitrarily small. Proposition~\ref{prop:lambda-bound} cannot be interpreted as proving that the complete relaxed trajectory remains feasible, however, in practice, we will see in the subsequent paragraph \ref{gumbel-soft-relax} that Gumbel-softmax activations often yield near-discrete bits, and this proposition thus explains why near-discrete solutions encountered on the trajectory remain feasible thanks to the Lagrangian penalty.

\subsection{Autoregressive relaxed rollout}
Phase~2 differentiates through the complete generated sequence. Using Eq.~\ref{eq:relaxed_embedding}, at step $i$ the model computes $\ell_i=T_\theta(e(\tilde{x}_1),\ldots,e(\tilde{x}_{i-1}))$ and $\tilde{x}_i=\pi_{i,1}$ from Eq.~\ref{eq:gumbel_softmax}. This recursion defines a differentiable map $\theta\mapsto\tilde{x}_\theta$, so the relaxed objective can be backpropagated through the entire autoregressive computational graph.
This dependence across positions is essential: changing the parameters that affect an early relaxed token also changes the prefix consumed at all subsequent generation steps. The resulting Jacobian therefore couples coordinates through the shared autoregressive computation and does not treat the $n$ relaxed variables as independent optimization parameters.

\subsection{Near-discrete exploration with Gumbel-softmax}\label{gumbel-soft-relax}
Standard softmax optimization can traverse highly fractional vectors with no combinatorial meaning: a fractional point may have high relaxed objective value while corresponding to no good feasible binary solution. Gumbel--softmax biases Phase~2 toward near-binary configurations. This has two practical implications. First, objective values encountered during training are evaluated near the combinatorial vertices. Second, the trajectory can make relatively sharp moves between near-discrete configurations rather than slowly traversing the interior of $[0,1]^n$. In the binary case, with logit difference $\Delta=\ell_1-\ell_0$ and Gumbel noise difference $G=g_1-g_0$, we have
$\tilde{x}_\tau=\sigma((\Delta+G)/\tau)$ where $\sigma$ is the logistic function.
\begin{proposition}[Low-temperature Gumbel--softmax]\label{prop:gumbel-gradient}
Let $D_\tau=\partial\tilde{x}_\tau/\partial\Delta$ and let $q$ be the standard logistic density. For finite $\Delta$, as $\tau\to0$,
\begin{equation}
\tilde{x}_\tau\to\mathbf 1\{\Delta+G>0\}\ \text{almost surely},\qquad
\E[D_\tau]\to q(-\Delta),\qquad
\Var(D_\tau)=\Theta(\tau^{-1}),
\label{eq:gumbel_moments}
\end{equation}
with $\E[D_\tau^2]=q(-\Delta)/(6\tau)+o(\tau^{-1})$.
\end{proposition}
Thus decreasing $\tau$ confines the forward transition to an $O(\tau)$-wide region while preserving a finite expected gradient through increasingly rare $O(\tau)$ switching events carrying $O(\tau^{-1})$ gradients; the diverging variance motivates keeping $\tau$ bounded away from zero. In our autoregressive model, switching an early coordinate also changes the conditioning prefix of all later decisions, thus flipping multiple bits in one gradient step. The proof is in Appendix \ref{gumbel-gradient}.

\subsection{Why optimize in transformer space?}\label{sec:induced_geometry}
The central motivation for \ADIP{} is the nonlinear change of variables $\theta\mapsto\tilde{x}_\theta\mapsto F_\lambda(\tilde{x}_\theta)$, where $F_\lambda(\tilde{x})=\tilde{x}^\top Q\tilde{x}-\lambda\Phi(\tilde{x})$. Let $J_\theta=\partial\tilde{x}_\theta/\partial\theta$.
\begin{theorem}[Induced decision-space dynamics]\label{thm:induced-flow}
For continuous-time gradient ascent $\dot\theta=\nabla_\theta F_\lambda(\tilde{x}_\theta)$,
\begin{equation}
\dot{\tilde{x}}=K_\theta\nabla_{\tilde{x}}F_\lambda(\tilde{x}),\qquad K_\theta:=J_\theta J_\theta^\top\succeq0.
\label{eq:induced_kernel_flow}
\end{equation}
\end{theorem}
\begin{proof}
By the chain rule, $\nabla_\theta F_\lambda=J_\theta^\top\nabla_{\tilde{x}}F_\lambda$ and $\dot{\tilde{x}}=J_\theta\dot\theta$, giving Eq.~\ref{eq:induced_kernel_flow}. Moreover $v^\top K_\theta v=\|J_\theta^\top v\|_2^2\ge0$.
\end{proof}
Direct relaxed gradient ascent follows $\dot{\tilde{x}}_{\mathrm{direct}}=\nabla_{\tilde{x}}F_\lambda(\tilde{x})$, corresponding to $K=I$, whereas \ADIP{} can rotate, suppress, or amplify directions through the state-dependent kernel $K_\theta$, analogous to parameter-space kernel dynamics in neural networks \citep{jacot2018neural}.
 Theorem~\ref{thm:induced-flow} therefore makes precise why the transformer is more than a high-dimensional encoding of $\tilde{x}$. Direct optimization follows the Euclidean gradient of $F_\lambda$, whereas \ADIP{} follows that gradient after multiplication by a positive-semidefinite matrix determined by the current network Jacobian. Because $J_\theta$ changes during training, the induced kernel is itself state-dependent.

The induced geometry also depends on the specific parameters, not only on the represented solution. For one attention head with input representation $H$, let $Q=HW_Q$ (queries), $K=HW_K$ (keys), and $V=HW_V$ (values). Causally masked attention is $\operatorname{softmax}(QK^\top/\sqrt{d_k}+M)V$, so the forward map depends on $W_Q,W_K$ through their product while gradient descent acts on the matrices separately.
\begin{theorem}[(informal) Path-dependent induced geometry]\label{thm:path-dependent-informal}
Assume an attention head affects the relaxed rollout nontrivially. For $c>0$, replace $W_Q$ by $cW_Q$ and $W_K$ by $c^{-1}W_K$, leaving all other parameters fixed. The resulting transformer defines exactly the same autoregressive map for every prefix, yet for some $c\ne1$ its kernel $J_\theta J_\theta^\top$ differs. Moreover, there exists a quadratic objective for which the two functionally identical transformers induce different instantaneous velocities in relaxed decision space.
\end{theorem}
In Appendix \ref{path-dependent}, we state a formal version of the above theorem and provide a full proof. The key symmetry is that $W_QW_K^\top$ is invariant under this rescaling, while the corresponding Jacobian blocks scale as $J_Q\mapsto c^{-1}J_Q$ and $J_K\mapsto cJ_K$. Thus the weights retain geometric information invisible in $\tilde{x}$ itself. Theorem~\ref{thm:path-dependent-informal} establishes that one can construct rescaled models that are functionally indistinguishable, producing the same conditional distributions at every autoregressive step, yet gradient ascent can evolve differently because it acts on the parameterization rather than on the represented function alone. In particular, the \ADIP{} trajectory cannot in general be described by an autonomous vector field depending only on the current relaxed solution: two parameterizations can represent exactly the same autoregressive map and nevertheless induce different local optimization dynamics. Motivated by this result, we call a trajectory ``tunneling-like'' if it reaches a point whose endpoint under direct relaxed gradient ascent belongs to a different attraction basin from that of the initial incumbent. This operational notion is tested in Section~\ref{sec:tunneling_experiments}.

\section{Experiments}
We evaluate whether \ADIP{} improves solver incumbents under matched wall-clock budgets and whether its weight-space trajectory accesses different basins of the relaxed objective.
The experiments therefore address two complementary questions: whether the method provides practical primal improvement over state-of-the-art open-source solvers given the same total time, and whether the proposed reparametrization changes the local optimization behavior in the relaxed landscape.

\subsection{Quadratic knapsack benchmark and protocol}
Following \citet{lopezpiqueres2025constraining}, we consider
\begin{equation}
\max_{x\in\{0,1\}^n}x^\top Qx\qquad\text{s.t.}\qquad w^\top x\le C,
\label{eq:qkp}
\end{equation}
where $w\in\R_+^n$ are item weights and $C>0$ is the capacity. In our experiments, $Q$ is a random non-symmetric integer-valued matrix, $Q_{ij}\sim\mathrm{Unif}\{-5,\ldots,5\}$, $w_i\sim\mathrm{Unif}\{0,\ldots,5\}$, $C=\frac14\sum_iw_i$, $n\in\{400,1000, 10000\}$, and we generate $100$ instances per experiment for size $n = 400$, and $10$ instances per experiment for sizes $n \in \{1000, 10000\}$. Our baselines are \SCIP{} \citep{achterberg2009scip,hojny2025scip10}, HiGHS \citep{huangfu2018highs}, CP-SAT \citep{perron2023ortools}, and CBC \citep{forrest2005cbc,cbc2026}. Only \SCIP{} can natively run quadratic integer programs, the other solvers require performing Fortet linearization \citep{fortet1960applications} leading to integer linear programs with $O (n^2)$ variables, which may become impractical already at sizes $n = 1000$. Each solver runs for $T_{\mathrm{seed}} = 500$ seconds to provide $x^{\inc}$; \ADIP{} then trains for $T_{\mathrm{train}} = 500$ seconds. The matched baseline runs the solver alone for $T_{\mathrm{total}}=T_{\mathrm{seed}}+T_{\mathrm{train}}$, and we report
\begin{equation}
\mathrm{Imp}_{\mathrm{time}}(\%)=100\frac{f(x^{\ADIP})-f(x^{\mathrm{OSS}}_{T_{\mathrm{total}}})}{|f(x^{\mathrm{OSS}}_{T_{\mathrm{total}}})|}.
\label{eq:time_improvement}
\end{equation}

The default model is a two-layer causal transformer with embedding dimension $64$, four heads, and feed-forward dimension $256$. Phase~1 uses $K_1=150$, $\eta_1=10^{-3}$, and $\varepsilon=0.001$; Phase~2 uses $K_2=2000$, $\eta_2=10^{-3}$ for size $n = 400$ (respectively $\eta_2 = 10^{-4}$ for sizes $n \in \{1000, 10000\}$), $\lambda=100$ for size $n = 400$ (respectively $\lambda = 500$ for sizes $n \in \{1000, 10000\}$) and $\tau=0.5$. Full implementation details are given in Appendix \ref{app:experimental-details}. The code to reproduce the experiments is available at \url{https://github.com/ouns972/adip-paper}.

\subsection{Main results}\label{sec:main-exp}
Table~\ref{tab:main_results} reports the matched-budget comparison. We find that \ADIP{} outperforms \SCIP{} in $58\%$ of the runs across the $n=400$ experiments, $90 \%$ of the runs across the $n=1000$ experiments, and $100\%$ of the runs across the $n= 10 000$ experiments, suggesting that \ADIP{} is stronger when the solvers leave greater room for improvement (at large problem sizes).
\begin{table}[t]
\centering
\caption{Matched-budget results on random QKP instances. Means and standard deviations (between parentheses) are computed across 100 runs per experiment for size $n=400$ and 10 runs per experiment for sizes $n \in \{1000, 10000\}$.}
\label{tab:main_results}
\small
\begin{tabular}{llccc}
\toprule
$n$ & Base solver & Matched solver obj. & \ADIP{} obj. & Improvement (\%)\\
\midrule
400 & \SCIP{} & $2080.51\,(339.71)$ & $2145.45\,(338.41)$ & $5.71\,(4.57)$\\
1000 & \SCIP{} & $5081.30\,(562.45)$ & $5710.90\,(620.19)$ & $12.78\,(10.29)$\\
10000 & \SCIP{} & $0.00\,(0.00)$ & $60.21\,(26.01)$ & $\infty$\\
400 & HiGHS & $226.00\,(240.42)$ & $893.00\,(796.20)$ & $378.48\,(156.70)$\\
1000 & HiGHS & $0.00\,(0.00)$ & $300.00\,(323.85)$ & $\infty$\\
400 & CP-SAT & $726.00\,(234.76)$ & $1277.50\,(355.67)$ & $77.31\,(8.35)$\\
1000 & CP-SAT & $1317.50\,(106.77)$ & $2687.50\,(584.78)$ & $102.85\,(27.95)$\\
400 & CBC & $19.50\,(13.44)$ & $133.50\,(21.92)$ & $848.45\,(765.87)$\\
1000 & CBC & $0.00\,(0.00)$ & $300.00\,(323.85)$ & $\infty$\\
\bottomrule
\end{tabular}
\end{table}

Figure~\ref{fig:training_curve} shows a representative objective-value trajectory. Starting from a \SCIP{} incumbent, the best feasible decoded solution improves during transformer training. The staircase structure reflects improvements of the discrete incumbent rather than the continuous relaxed training objective.
\begin{figure}[t]
\centering
\includegraphics[width=0.68\linewidth]{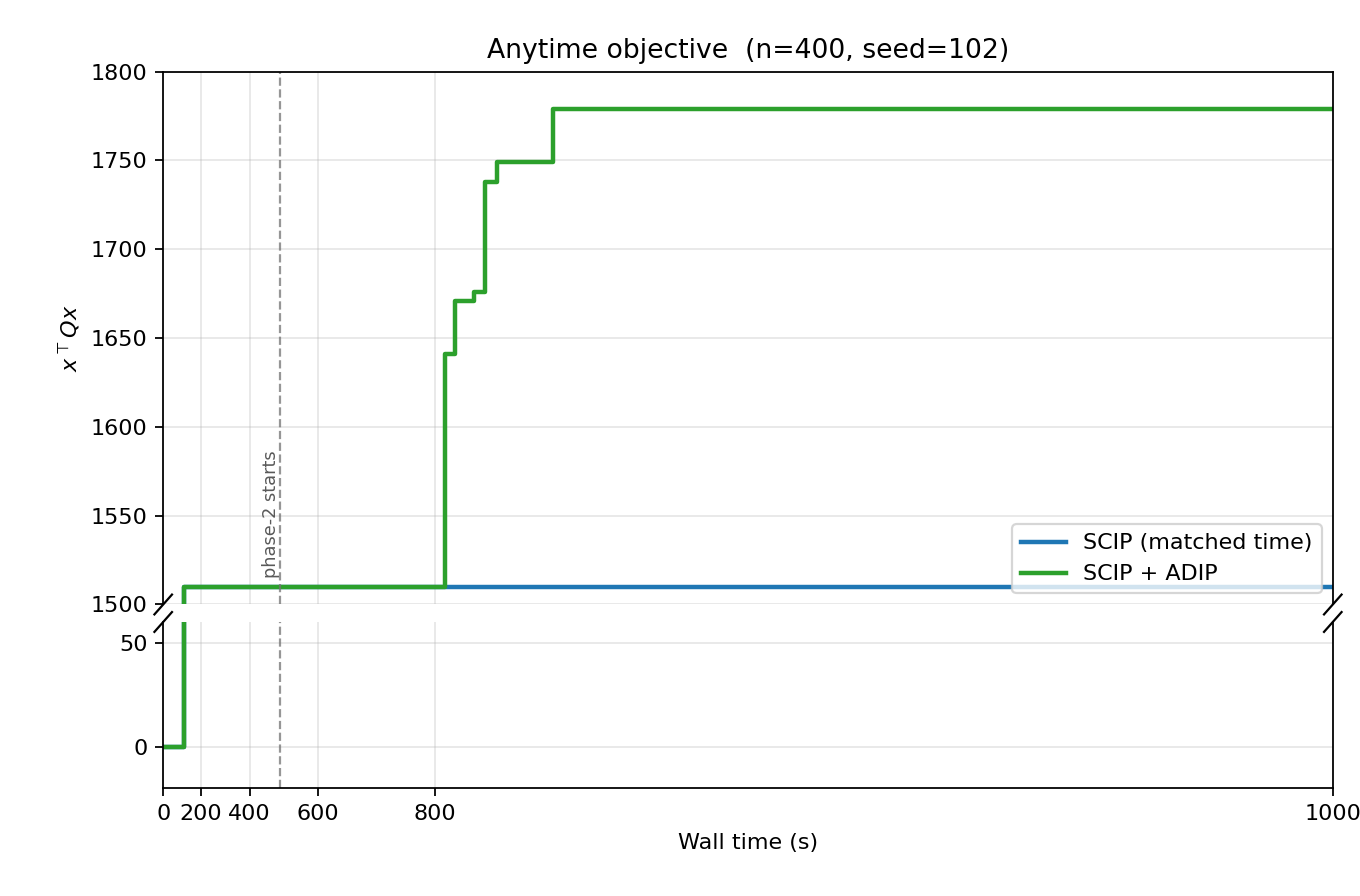}
\caption{Representative objective-value trajectory: initial incumbent search and hard-decoded objective over ADIP training.}
\label{fig:training_curve}
\end{figure}
The improvement also correlates with the \SCIP{} gap at $T_{\mathrm{seed}}$; this secondary diagnostic is shown in Appendix \ref{app:experimental-details}. This observation is consistent with the intuition that the neural improvement procedure is most useful when substantial optimization remains after the seed phase, although the gap alone should not be interpreted as determining \ADIP{} performance.

\subsection{Barrier-crossing and tunneling-like behavior}\label{sec:tunneling_experiments}
We next investigate whether the transformer reparametrization accesses regions of the relaxed landscape that are difficult to reach by combinatorial optimization  or by direct optimization in $\tilde{x}$-space. Let $E(\tilde{x})=-\tilde{x}^\top Q\tilde{x}$ denote the energy associated with the maximization objective. For an initial relaxed point $\tilde{x}^{(0)}$, we run direct gradient descent on $E$ over the continuous feasible relaxation,
\begin{equation}
\tilde{x}^{(t+1)}=\tilde{x}^{(t)}-\eta\nabla_{\tilde{x}}E(\tilde{x}^{(t)}),
\label{eq:basin_descent}
\end{equation}
checking $A\tilde{x}\le b$ at each step. We apply this procedure to two types of initial points: the original solver incumbent $x^{\inc}$, viewed as a point in $[0,1]^n$, and relaxed transformer outputs $\tilde{x}^{(k)}_\theta$ encountered during Phase~2. We iterate until convergence or until the feasible boundary prevents further improvement. If these initial points converge to different endpoints, we can infer that they belong to different attraction basins of the direct relaxed dynamics. The comparison therefore asks whether the transformer trajectory has entered a different region of the non-convex relaxed landscape.

Empirically, \ADIP{} visits points whose relaxed-gradient endpoints differ from that of the initial incumbent, indicating that optimization in transformer parameter space can cross basin boundaries not crossed by direct local optimization from the incumbent. This experiment directly probes the geometric claim of Theorems~\ref{thm:induced-flow}--\ref{thm:path-dependent-informal}: if the transformer merely reproduced ordinary Euclidean optimization in $\tilde{x}$-space, points generated from the incumbent would not be expected to systematically acquire different direct-gradient endpoints. We interpret this as a tunneling-like effect induced by the autoregressive reparametrization. Figure~\ref{fig:tsne_tunneling} visualizes the trajectories; t-SNE is used only for visualization, while basin identity is determined from endpoints in the original high-dimensional relaxed space.
\begin{figure}[t]
\centering
\includegraphics[width=0.72\linewidth]{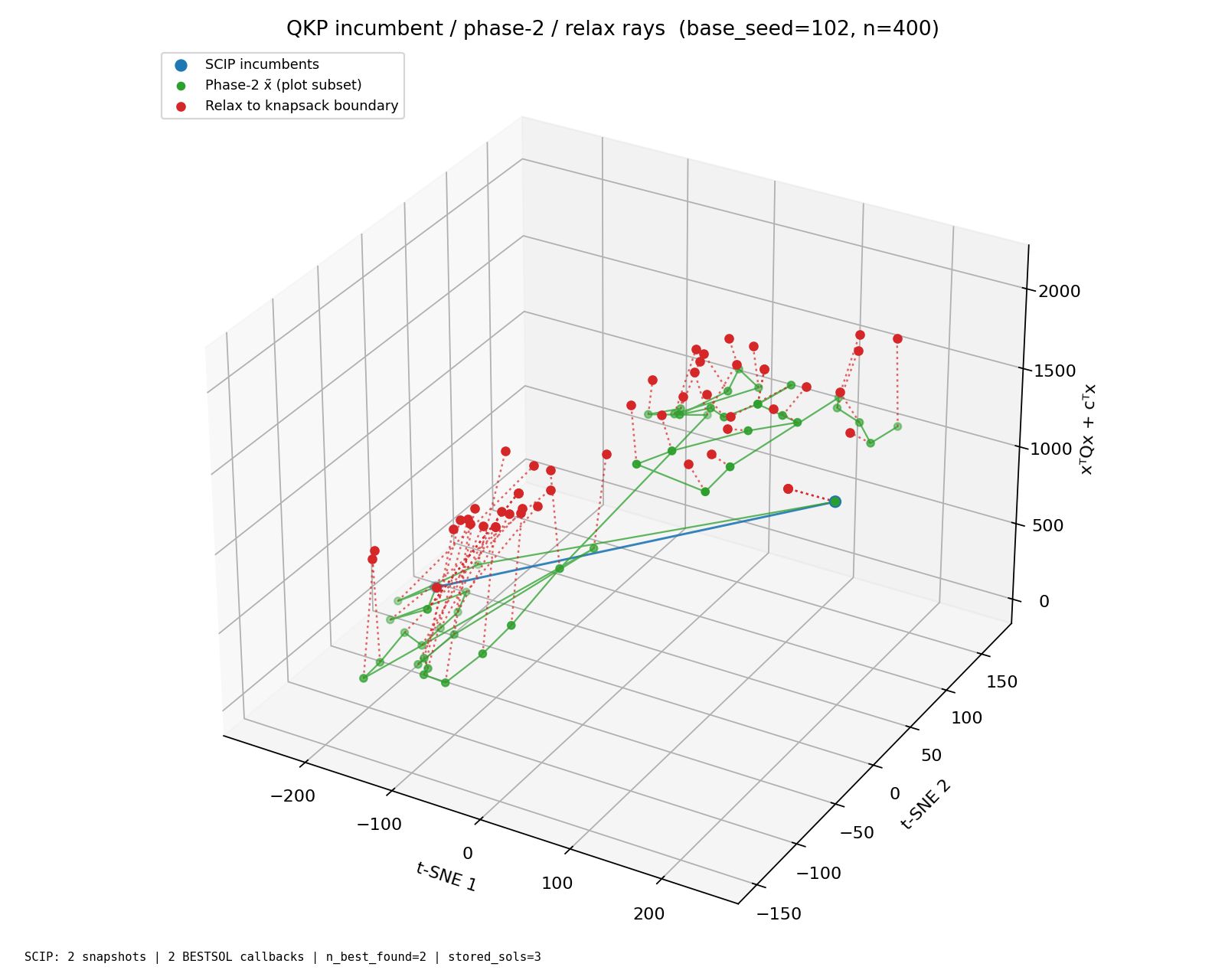}
\caption{Visualization of the tunneling-like experiment. The solver incumbents, relaxed \ADIP{} trajectory, and direct-relaxation endpoints are projected with t-SNE; the vertical axis shows objective value. Basin identity is determined in the original space.}
\label{fig:tsne_tunneling}
\end{figure}
\begin{table}[t]
\centering
\caption{Representative basin analysis of the relaxed landscape.}
\label{tab:basin_analysis}
\small
\begin{tabular}{lcccc}
\toprule
$n$ & Trajectory points & Distinct endpoints & Incumbent endpoint & Best endpoint\\
\midrule
400 & 73 & 45 & 1931.3 & 1971.8\\
400 & 72 & 52 & 1842.5 & 2020.4\\
1000 & 90 & 35 & $4281.0$ & $5083.0$ \\
1000 & 88 & 29 & $5248.0$ & $5874.1$ \\
\bottomrule
\end{tabular}
\end{table}

Overall, the experiments show that \ADIP{} can improve solver incumbents under
matched computational budgets and that optimization in transformer parameter
space can reach regions of the relaxed landscape that are not reached by combinatorial solvers or direct
optimization from the solver incumbent. We therefore view \ADIP{} as a
complementary primal-improvement mechanism.

\section{Limitations}

\ADIP{} is trained separately for each instance and requires a feasible
incumbent from an external solver. The Lagrangian penalty does not guarantee
feasibility of every relaxed intermediate point, although Gumbel-softmax exploration remains close to the original combinatorial structure and the algorithm always
returns the best feasible hard solution encountered. Our tunneling analogy is
geometric and empirical and does not imply global optimization guarantees.
Finally, our experiments focus on random QKP instances; evaluating the method
on larger and more diverse integer-programming benchmarks, as well as
amortizing it across related instances, are important directions for future
work.

\section{Conclusion}

We introduced \ADIP{}, which uses an autoregressive transformer as a nonlinear
reparametrization of relaxed binary decision variables for primal improvement
in integer programming. Optimizing through the transformer induces a
state-dependent geometry in decision space, providing a search mechanism that
differs from combinatorial optimization and from direct optimization of the relaxed variables. On random QKP
instances, \ADIP{} improves state-of-the-art solver incumbents under matched computational
budgets and accesses distinct attraction basins of the relaxed landscape.
Together, these results suggest that autoregressive neural reparametrization can
provide a complementary search mechanism for hard discrete optimization.

\newpage

\bibliography{iclr2027_conference}

@incollection{gallo1980quadratic,
  author    = {Gallo, Giorgio and Hammer, Peter L. and Simeone, Bruno},
  title     = {Quadratic Knapsack Problems},
  booktitle = {Combinatorial Optimization},
  editor    = {Padberg, Manfred W.},
  series    = {Mathematical Programming Studies},
  volume    = {12},
  pages     = {132--149},
  publisher = {Springer},
  address   = {Berlin, Heidelberg},
  year      = {1980},
  doi       = {10.1007/BFb0120892}
}

@article{pisinger2007solution,
  author  = {Pisinger, W. David and Rasmussen, Anders Bo and Sandvik, Rune},
  title   = {Solution of Large Quadratic Knapsack Problems Through Aggressive Reduction},
  journal = {INFORMS Journal on Computing},
  volume  = {19},
  number  = {2},
  pages   = {280--290},
  year    = {2007},
  doi     = {10.1287/ijoc.1050.0172}
}

@inproceedings{jang2017categorical,
  author        = {Jang, Eric and Gu, Shixiang and Poole, Ben},
  title         = {Categorical Reparameterization with Gumbel-Softmax},
  booktitle     = {International Conference on Learning Representations},
  year          = {2017},
  eprint        = {1611.01144},
  archivePrefix = {arXiv},
  primaryClass  = {stat.ML},
  url           = {https://arxiv.org/abs/1611.01144}
}

@inproceedings{maddison2017concrete,
  author        = {Maddison, Chris J. and Mnih, Andriy and Teh, Yee Whye},
  title         = {The Concrete Distribution: A Continuous Relaxation of Discrete Random Variables},
  booktitle     = {International Conference on Learning Representations},
  year          = {2017},
  eprint        = {1611.00712},
  archivePrefix = {arXiv},
  primaryClass  = {cs.LG},
  url           = {https://arxiv.org/abs/1611.00712}
}

@inproceedings{vaswani2017attention,
  author    = {Vaswani, Ashish and Shazeer, Noam and Parmar, Niki and Uszkoreit, Jakob and Jones, Llion and Gomez, Aidan N. and Kaiser, {\L}ukasz and Polosukhin, Illia},
  title     = {Attention Is All You Need},
  booktitle = {Advances in Neural Information Processing Systems},
  volume    = {30},
  pages     = {5998--6008},
  publisher = {Curran Associates, Inc.},
  year      = {2017}
}

@article{achterberg2009scip,
  author  = {Achterberg, Tobias},
  title   = {{SCIP}: Solving Constraint Integer Programs},
  journal = {Mathematical Programming Computation},
  volume  = {1},
  number  = {1},
  pages   = {1--41},
  year    = {2009},
  doi     = {10.1007/s12532-008-0001-1}
}

@book{nemhauser1988integer,
  author    = {Nemhauser, George L. and Wolsey, Laurence A.},
  title     = {Integer and Combinatorial Optimization},
  publisher = {John Wiley \& Sons},
  address   = {New York, NY},
  year      = {1988},
  isbn      = {9780471828198},
  doi       = {10.1002/9781118627372}
}

@article{belotti2013mixed,
  author  = {Belotti, Pietro and Kirches, Christian and Leyffer, Sven
             and Linderoth, Jeff and Luedtke, James and Mahajan, Ashutosh},
  title   = {Mixed-Integer Nonlinear Optimization},
  journal = {Acta Numerica},
  volume  = {22},
  pages   = {1--131},
  year    = {2013},
  doi     = {10.1017/S0962492913000032}
}

@article{bertsimas2016bestsubset,
  author  = {Bertsimas, Dimitris and King, Angela and Mazumder, Rahul},
  title   = {Best Subset Selection via a Modern Optimization Lens},
  journal = {The Annals of Statistics},
  volume  = {44},
  number  = {2},
  pages   = {813--852},
  year    = {2016},
  doi     = {10.1214/15-AOS1388}
}

@article{bixby2012brief,
  author  = {Bixby, Robert E.},
  title   = {A Brief History of Linear and Mixed-Integer Programming
             Computation},
  journal = {Documenta Mathematica},
  volume  = {Extra Volume ISMP},
  pages   = {107--121},
  year    = {2012},
  doi     = {10.4171/DMS/6/16}
}

@incollection{lodi2010mixed,
  author    = {Lodi, Andrea},
  title     = {Mixed Integer Programming Computation},
  booktitle = {50 Years of Integer Programming 1958--2008:
               From the Early Years to the State-of-the-Art},
  editor    = {J{\"u}nger, Michael and Liebling, Thomas M. and Naddef, Denis
               and Nemhauser, George L. and Pulleyblank, William R.
               and Reinelt, Gerhard and Rinaldi, Giovanni
               and Wolsey, Laurence A.},
  pages     = {619--645},
  publisher = {Springer},
  address   = {Berlin, Heidelberg},
  year      = {2010},
  doi       = {10.1007/978-3-540-68279-0_16}
}

@article{gleixner2021miplib,
  author  = {Gleixner, Ambros and Hendel, Gregor and Gamrath, Gerald
             and Achterberg, Tobias and Bastubbe, Michael and Berthold, Timo
             and Christophel, Philipp M. and Jarck, Kati and Koch, Thorsten
             and Linderoth, Jeff and L{\"u}bbecke, Marco E.
             and Mittelmann, Hans D. and Ozyurt, Derya and Ralphs, Ted K.
             and Salvagnin, Domenico and Shinano, Yuji},
  title   = {{MIPLIB} 2017: Data-Driven Compilation of the 6th
             Mixed-Integer Programming Library},
  journal = {Mathematical Programming Computation},
  volume  = {13},
  number  = {3},
  pages   = {443--490},
  year    = {2021},
  doi     = {10.1007/s12532-020-00194-3}
}

@article{berthold2013measuring,
  author  = {Berthold, Timo},
  title   = {Measuring the Impact of Primal Heuristics},
  journal = {Operations Research Letters},
  volume  = {41},
  number  = {6},
  pages   = {611--614},
  year    = {2013},
  doi     = {10.1016/j.orl.2013.08.007}
}

@article{fischetti2003local,
  author  = {Fischetti, Matteo and Lodi, Andrea},
  title   = {Local Branching},
  journal = {Mathematical Programming},
  volume  = {98},
  number  = {1--3},
  pages   = {23--47},
  year    = {2003},
  doi     = {10.1007/s10107-003-0395-5}
}

@article{rothberg2007evolutionary,
  author  = {Rothberg, Edward},
  title   = {An Evolutionary Algorithm for Polishing Mixed Integer
             Programming Solutions},
  journal = {INFORMS Journal on Computing},
  volume  = {19},
  number  = {4},
  pages   = {534--541},
  year    = {2007},
  doi     = {10.1287/ijoc.1060.0189}
}

@inproceedings{geng2025differentiable,
  title     = {Differentiable Integer Linear Programming},
  author    = {Geng, Zijie and Wang, Jie and Li, Xijun and Zhu, Fangzhou and Hao, Jianye and Li, Bin and Wu, Feng},
  booktitle = {The Thirteenth International Conference on Learning Representations},
  year      = {2025},
  url       = {https://openreview.net/forum?id=FPfCUJTsCn}
}

@article{nair2020solving,
  title   = {Solving Mixed Integer Programs Using Neural Networks},
  author  = {Nair, Vinod and Bartunov, Sergey and Gimeno, Felix and von Glehn, Ingrid and Lichocki, Pawel and Lobov, Ivan and O'Donoghue, Brendan and Sonnerat, Nicolas and Tjandraatmadja, Christian and Wang, Pengming and Addanki, Ravichandra and Hapuarachchi, Tharindi and Keck, Thomas and Keeling, James and Kohli, Pushmeet and Ktena, Ira and Li, Yujia and Vinyals, Oriol and Zwols, Yori},
  journal = {arXiv preprint arXiv:2012.13349},
  year    = {2020},
  url     = {https://arxiv.org/abs/2012.13349}
}

@inproceedings{khalil2022mipgnn,
  title     = {{MIP-GNN}: A Data-Driven Framework for Guiding Combinatorial Solvers},
  author    = {Khalil, Elias B. and Morris, Christopher and Lodi, Andrea},
  booktitle = {Proceedings of the AAAI Conference on Artificial Intelligence},
  volume    = {36},
  number    = {9},
  pages     = {10219--10227},
  year      = {2022}
}

@inproceedings{han2023gnn,
  title     = {A {GNN}-Guided Predict-and-Search Framework for Mixed-Integer Linear Programming},
  author    = {Han, Qingyu and Yang, Linxin and Chen, Qian and Zhou, Xiang and Zhang, Dong and Wang, Akang and Sun, Ruoyu and Luo, Xiaodong},
  booktitle = {The Eleventh International Conference on Learning Representations},
  year      = {2023}
}

@article{dipillo1989exact,
  title   = {Exact Penalty Functions in Constrained Optimization},
  author  = {Di Pillo, Gianni and Grippo, Luigi},
  journal = {SIAM Journal on Control and Optimization},
  volume  = {27},
  number  = {6},
  pages   = {1333--1360},
  year    = {1989}
}

@article{bello2017neural,
  title   = {Neural Combinatorial Optimization with Reinforcement Learning},
  author  = {Bello, Irwan and Pham, Hieu and Le, Quoc V. and Norouzi, Mohammad and Bengio, Samy},
  journal = {International Conference on Learning Representations Workshop},
  year    = {2017},
  url     = {https://arxiv.org/abs/1611.09940}
}

@inproceedings{vinyals2015pointer,
  title     = {Pointer Networks},
  author    = {Vinyals, Oriol and Fortunato, Meire and Jaitly, Navdeep},
  booktitle = {Advances in Neural Information Processing Systems},
  volume    = {28},
  year      = {2015}
}

@misc{sornwanee2026diffilo,
  title         = {Differentiable Integer Linear Programming is not Differentiable \& it's not a mere technical problem},
  author        = {Sornwanee, Thanawat},
  year          = {2026},
  eprint        = {2601.17800},
  archivePrefix = {arXiv},
  primaryClass  = {math.OC},
  doi           = {10.48550/arXiv.2601.17800},
  url           = {https://arxiv.org/abs/2601.17800}
}

@inproceedings{tang2025autoregressive,
  title     = {Autoregressive RL Approach for Mixed-Integer Linear Programs},
  author    = {Tang, Paul Mingzheng and Lee, Moses Hong De and Lau, Hoong Chuin},
  booktitle = {Learning and Intelligent Optimization},
  pages     = {1--17},
  year      = {2025},
  publisher = {Springer},
  series    = {Lecture Notes in Computer Science},
  volume    = {15745},
  doi       = {10.1007/978-3-032-09192-5_1}
}

@article{hibatallah2021variational,
  title   = {Variational Neural Annealing},
  author  = {Hibat-Allah, Mohamed and Inack, Estelle M. and Wiersema, Roeland and Melko, Roger G. and Carrasquilla, Juan},
  journal = {Nature Machine Intelligence},
  volume  = {3},
  pages   = {952--961},
  year    = {2021}
}

@article{lopezpiqueres2025constraining,
  title     = {Cons-training Tensor Networks: Embedding and Optimization over Discrete Linear Constraints},
  author    = {Lopez Piqueres, Javier and Chen, Jing},
  journal   = {SciPost Physics},
  volume    = {18},
  pages     = {192},
  year      = {2025},
  publisher = {SciPost},
  doi       = {10.21468/SciPostPhys.18.6.192}
}

@misc{nakaji2024generative,
  title         = {The Generative Quantum Eigensolver ({GQE}) and Its Application for Ground State Search},
  author        = {Nakaji, Kouhei and Kristensen, Lasse Bj{\o}rn and Campos-Gonzalez-Angulo, Jorge A. and Vakili, Mohammad Ghazi and Huang, Haozhe and Bagherimehrab, Mohsen and Gorgulla, Christoph and Wong, FuTe and McCaskey, Alex and Kim, Jin-Sung and Nguyen, Thien and Rao, Pooja and others},
  year          = {2024},
  eprint        = {2401.09253},
  archivePrefix = {arXiv},
  primaryClass  = {quant-ph},
  doi           = {10.48550/arXiv.2401.09253}
}

@inproceedings{jacot2018neural,
  title     = {Neural Tangent Kernel: Convergence and Generalization in Neural Networks},
  author    = {Jacot, Arthur and Gabriel, Franck and Hongler, Cl{\'e}ment},
  booktitle = {Advances in Neural Information Processing Systems},
  volume    = {31},
  year      = {2018}
}

@article{fortet1960applications,
  author  = {Fortet, Robert},
  title   = {Applications de l'alg{\`e}bre de Boole en recherche op{\'e}rationnelle},
  journal = {Revue Fran{\c{c}}aise de Recherche Op{\'e}rationnelle},
  volume  = {4},
  number  = {14},
  pages   = {17--26},
  year    = {1960}
}

@book{martello1990knapsack,
  author    = {Martello, Silvano and Toth, Paolo},
  title     = {Knapsack Problems: Algorithms and Computer Implementations},
  publisher = {John Wiley \& Sons},
  address   = {Chichester},
  year      = {1990}
}

@article{goemans1995improved,
  author  = {Goemans, Michel X. and Williamson, David P.},
  title   = {Improved Approximation Algorithms for Maximum Cut and
             Satisfiability Problems Using Semidefinite Programming},
  journal = {Journal of the ACM},
  volume  = {42},
  number  = {6},
  pages   = {1115--1145},
  year    = {1995},
  doi     = {10.1145/227683.227684}
}

@incollection{cornuejols1990uncapacitated,
  author    = {Cornu{\'e}jols, G{\'e}rard and Nemhauser, George L. and Wolsey, Laurence A.},
  title     = {The Uncapacitated Facility Location Problem},
  booktitle = {Discrete Location Theory},
  editor    = {Mirchandani, P. B. and Francis, R. L.},
  publisher = {John Wiley \& Sons},
  address   = {New York},
  pages     = {119--171},
  year      = {1990}
}

@book{pinedo2016scheduling,
  author    = {Pinedo, Michael L.},
  title     = {Scheduling: Theory, Algorithms, and Systems},
  edition   = {5},
  publisher = {Springer},
  year      = {2016},
  doi       = {10.1007/978-3-319-26580-3}
}

@misc{hojny2025scip10,
  title         = {The {SCIP} Optimization Suite 10.0},
  author        = {Hojny, Christopher and Besan{\c{c}}on, Mathieu and Bestuzheva, Ksenia
                   and Borst, Sander and Dion{\'i}sio, Jo{\~a}o and Ehls, Johannes
                   and Eifler, Leon and Ghannam, Mohammed and Gleixner, Ambros
                   and G{\"o}{\ss}, Adrian and Hoen, Alexander
                   and von Holly-Ponientzietz, Jacob and van der Hulst, Rolf
                   and Kamp, Dominik and Koch, Thorsten and Kofler, Kevin
                   and Lentz, Jurgen and L{\"u}bbecke, Marco and Maher, Stephen J.
                   and Meinhold, Paul Matti and Mexi, Gioni and Mohr, Til
                   and M{\"u}hmer, Erik and Patel, Krunal Kishor and Pfetsch, Marc E.
                   and Pokutta, Sebastian and Reinartz Groba, Chantal
                   and Serrano, Felipe and Shinano, Yuji and Turner, Mark
                   and Vigerske, Stefan and Walter, Matthias and Weninger, Dieter
                   and Xu, Liding},
  year          = {2025},
  eprint        = {2511.18580},
  archivePrefix = {arXiv},
  primaryClass  = {math.OC},
  doi           = {10.48550/arXiv.2511.18580}
}

@inproceedings{yun2020transformers,
  title     = {Are Transformers Universal Approximators of Sequence-to-Sequence Functions?},
  author    = {Yun, Chulhee and Bhojanapalli, Srinadh and Rawat, Ankit Singh and Reddi, Sashank J. and Kumar, Sanjiv},
  booktitle = {International Conference on Learning Representations},
  year      = {2020},
  url       = {https://openreview.net/forum?id=ByxRM0Ntvr}
}

@article{huangfu2018highs,
  author  = {Qi Huangfu and J. A. J. Hall},
  title   = {Parallelizing the dual revised simplex method},
  journal = {Mathematical Programming Computation},
  volume  = {10},
  number  = {1},
  pages   = {119--142},
  year    = {2018},
  doi     = {10.1007/s12532-017-0130-5}
}

@incollection{forrest2005cbc,
  author    = {John Forrest and Robin Lougee-Heimer},
  title     = {{CBC} User Guide},
  booktitle = {Tut{OR}ials in Operations Research: Emerging Theory, Methods, and Applications},
  publisher = {INFORMS},
  year      = {2005},
  pages     = {257--277},
  doi       = {10.1287/educ.1053.0020}
}

@software{cbc2026,
  author       = {{COIN-OR}},
  title        = {{CBC}: COIN-OR Branch-and-Cut Solver},
  version      = {2.10.13},
  year         = {2026},
  url          = {https://github.com/coin-or/Cbc},
  note         = {Open-source mixed-integer programming solver}
}

@inproceedings{perron2023ortools,
  author    = {Perron, Laurent and Didier, Fr{\'e}d{\'e}ric and Gay, Steven},
  title     = {The {CP-SAT-LP} Solver},
  booktitle = {29th International Conference on Principles and Practice of
               Constraint Programming (CP 2023)},
  series    = {Leibniz International Proceedings in Informatics},
  volume    = {280},
  pages     = {3:1--3:2},
  year      = {2023},
  publisher = {Schloss Dagstuhl -- Leibniz-Zentrum f{\"u}r Informatik},
  doi       = {10.4230/LIPIcs.CP.2023.3}
}

@misc{leng2023quantum,
  title         = {Quantum Hamiltonian Descent},
  author        = {Leng, Jiaqi and Hickman, Ethan and Li, Joseph and Wu, Xiaodi},
  year          = {2023},
  eprint        = {2303.01471},
  archivePrefix = {arXiv},
  primaryClass  = {quant-ph},
  doi           = {10.48550/arXiv.2303.01471}
}

@inproceedings{chaudhary2025qihd,
  title     = {Quantum-Inspired Hamiltonian Descent for Mixed-Integer Quadratic Programming},
  author    = {Chaudhary, Shreya and Cheng, Jinglei and Kushnir, Samuel
               and Leng, Jiaqi and Liu, Pengyu and Peng, Yuxiang
               and Wang, Hanrui and Wu, Xiaodi},
  booktitle = {NeurIPS 2025 Workshop on GPU-Accelerated and Scalable Optimization (ScaleOPT)},
  year      = {2025}
}
\bibliographystyle{iclr2027_conference}

\newpage
\appendix
\section{Appendix}

\subsection{Additional details on related works}

Our experimental setup is directly inspired by the constrained tensor-network approach of \citet{lopezpiqueres2025constraining}, which introduces constrained matrix product states (MPS) whose block structure exactly restricts their support to binary configurations satisfying prescribed linear constraints. The authors optimize the resulting generative model using an unsupervised annealing procedure and evaluate it on the QKP, where they report improvements over SCIP on random instances. In fact, their QKP experiments form the basis of our experimental benchmark: following \citet{lopezpiqueres2025constraining}, we consider randomly generated integer-valued quadratic objectives and knapsack constraints and compare the quality of solutions obtained against SCIP under controlled computational budgets. The two approaches differ fundamentally, however, in how feasibility and optimization are represented. Constrained MPS encode the linear constraints exactly into the tensor-network architecture, so that samples have support only on the feasible set. This is a serious limitation of their method, as in general, exact encoding of linear constraints over binary variables via MPS requires exponentially many parameters. Our transformer does not hard-code the feasible set into its architecture; instead, it is first initialized on a feasible incumbent returned by SCIP and subsequently optimized using a differentiable objective together with a Lagrangian feasibility penalty. Moreover, while the constrained-MPS algorithm updates a probability distribution over feasible configurations through an annealed sample-and-train procedure, our method differentiates the relaxed QKP objective end-to-end through an autoregressive map $\theta\mapsto\tilde{x}_{\theta}$ and thereby performs optimization directly in the shared parameter space of the transformer. Thus, while our QKP benchmark and its motivation are explicitly inherited from \citet{lopezpiqueres2025constraining}, our method investigates a different mechanism for navigating the combinatorial landscape: differentiable autoregressive reparametrization, which has better scalability potential than exact constraint encoding via MPS.

The second conceptual connection is the Generative Quantum Eigensolver (GQE) of \citet{nakaji2024generative}. GQE departs from conventional variational quantum eigensolvers by placing the trainable parameters in a classical generative model rather than directly in a parameterized quantum circuit. In its GPT-QE realization, a transformer autoregressively generates sequences of quantum-circuit components, and the generative model is iteratively trained so that the circuits it produces attain lower energies. This perspective is closely related to the motivation of our method: in both cases, a difficult optimization problem is lifted from its original representation into the parameter space of a classical autoregressive generative model, whose shared parameters induce correlated changes in the generated object. The crucial difference is that GQE generates descriptions of quantum circuits whose energies must subsequently be evaluated, whereas our transformer directly generates the binary decision variables of the integer program through a differentiable relaxation. Consequently, our objective can be backpropagated through the generated variables to the transformer weights, making $\theta\mapsto\tilde{x}_{\theta}\mapsto f(\tilde{x}_{\theta})$ an explicit differentiable reparametrization of the original optimization landscape. Our method therefore shares GQE's broader quantum-inspired idea of optimizing through a classical generative representation, but applies it entirely classically and directly to constrained discrete optimization, with the additional use of solver-incumbent initialization and Gumbel--softmax relaxation to obtain a differentiable, near-discrete optimization trajectory.

A complementary connection is Quantum Hamiltonian Descent (QHD)
\citep{leng2023quantum}, which formulates continuous non-convex optimization
as a time-dependent quantum Hamiltonian evolution. QHD is motivated in part by
the ability of quantum dynamics to exploit trajectories that are inaccessible to
classical local descent, including tunneling across barriers in the objective
landscape. This motivation is closely related to our empirical tunneling analogy,
but the underlying mechanisms are fundamentally different. QHD requires a
quantum Hamiltonian evolution over the original continuous variables, whereas
ADIP is entirely classical: it changes the optimization geometry by lifting the
relaxed decision vector to the parameter space of an autoregressive transformer.
In particular, our induced dynamics
$\dot{\tilde{x}}=J_\theta J_\theta^\top\nabla_{\tilde{x}}F$
provide a classical, state-dependent reparametrization of the gradient field,
while the observed basin crossing arises from optimization through this neural
representation together with stochastic near-discrete exploration.

Recent concurrent work on Quantum-Inspired Hamiltonian Descent (QIHD)
\citep{chaudhary2025qihd} similarly transfers Hamiltonian intuition to a
classical optimizer for mixed-integer quadratic programs. QIHD explicitly
simulates momentum-position Hamiltonian dynamics with constraint penalties,
whereas ADIP obtains non-Euclidean decision-space dynamics implicitly through
the Jacobian of an autoregressive neural reparametrization.

\subsection{Experimental details}
\label{app:experimental-details}

\paragraph{Compute environment.}
All reported runs were executed on a single Apple MacBook Pro with an Apple M4~Max CPU
(14 cores: 10 performance + 4 efficiency) and 36\,GB unified memory.
Training used CPU-only PyTorch; no GPU acceleration was used as transformers used are small and GPU overheads dominate here.
Warm-start solvers were restricted to a single thread.
BLAS/OpenMP thread counts were likewise set to~1 for reproducibility.
Wall-clock timing was used for SCIP, HiGHs, CP-SAT and CBC with fixed random seeds.

\paragraph{Instance generation.}
As in the main text, for each size $n\in\{400,1000,10000\}$, for each experiment, we sample
$Q_{ij}\sim\mathrm{Unif}\{-5,\ldots,5\}$ (integer, nonsymmetric),
$w_i\sim\mathrm{Unif}\{0,\ldots,5\}$, and capacity
$C=\tfrac14\sum_i w_i$.
Instances are identified by a base seed. The generation code is released in the repository.

\paragraph{Warm-start models.}
SCIP solves the native nonconvex quadratic integer program (quadratic epigraph).
HiGHS and CBC receive a Fortet/McCormick linearization of the binary quadratic,
introducing an auxiliary binary $y_{ij}$ for each nonzero product term together with
$y_{ij}\le x_i$, $y_{ij}\le x_j$, $y_{ij}\ge x_i+x_j-1$.
CP-SAT uses the same auxiliaries but encodes $y_{ij}=x_i x_j$ via
OR-Tools \texttt{AddMultiplicationEquality} (equivalent on $\{0,1\}^n$).
For dense indefinite $Q$, the linearized models have $O (n^2)$ binaries and
constraints, which becomes severe already at $n=1000$ for LP-based MIP solvers;
CP-SAT's SAT/CP search often returns a feasible incumbent sooner under the same
500\,s budget.

\paragraph{Matched-budget protocol.}
Each base solver is allotted $T_{\mathrm{seed}}=500$\,s to produce $x^{\mathrm{inc}}$.
ADIP then runs with a Phase-2 wall-clock cap $T_{\mathrm{train}}=500$\,s
(and at most $K_2=2000$ Adam steps, with early stopping on plateau).
The matched baseline continues the same solver alone for
$T_{\mathrm{total}}=T_{\mathrm{seed}}+T_{\mathrm{train}}=1000$\,s.
We report relative improvement
\[
\mathrm{Imp}_{\mathrm{time}}(\%)
=
100\cdot
\frac{f(x^{\mathrm{ADIP}})-f(x^{\mathrm{OSS}}_{T_{\mathrm{total}}})}
{\lvert f(x^{\mathrm{OSS}}_{T_{\mathrm{total}}})\rvert},
\]
and write $\infty$ when the matched solver objective is identically zero while ADIP
returns a strictly positive feasible objective (common for linearized solvers at
large $n$, and for SCIP at $n=10000$ under our budgets).

\paragraph{Architecture and optimization.}
The model is a two-layer causal transformer over the binary alphabet with
embedding dimension $d=64$, $4$ heads, feed-forward width $d_{\mathrm{ff}}=256$,
and banded sparse attention retaining a fraction $0.1$ of previous positions.
Phase~1 minimizes the smoothed teacher-forcing CE~\eqref{eq:init_loss} with Adam,
$K_1=150$ maximum steps, learning rate $\eta_1=10^{-3}$,
plateau patience $P_1=12$ and minimum delta $\delta_1=10^{-5}$,
and target smoothing $\varepsilon=0.001$.
Phase~2 maximizes the soft Lagrangian~\eqref{eq:phase2_objective} with Adam,
Gumbel--softmax straight-through estimation at temperature $\tau=0.5$,
soft weight $w_{\mathrm{soft}}=1$, and plateau patience $P_2=45$
(minimum delta $\delta_2=10^{-5}$).
Size-dependent Phase~2 knobs are:
\begin{center}
\begin{tabular}{lcc}
\toprule
$n$ & Phase-2 LR $\eta_2$ & knapsack penalty $\lambda$ \\
\midrule
$400$ & $10^{-3}$ (same as $\eta_1$) & $100$ \\
$1000$, $10000$ & $10^{-4}$ & $500$ \\
\bottomrule
\end{tabular}
\end{center}
At every Phase~2 step we hard-decode $\hat x$, test knapsack feasibility, and retain
the best feasible objective encountered.
Phase-1 CE uses a one-pass implementation; Phase-2 rollouts use KV caching with
chunked gradient checkpointing to keep long sequences tractable.

\paragraph{Objective-value trajectory (Fig.~\ref{fig:training_curve}).}
Figure~\ref{fig:training_curve} compares the matched SCIP-alone curve over $1000$\,s with the
SCIP+$T_{\mathrm{train}}$ ADIP curve.
The plot uses a broken $y$-axis (near-zero band and high-objective zoom) and a
late-time expanding $x$-scale $-\log(T-t)$ to emphasize improvements near the end
of the budget.

\paragraph{Basin / tunneling diagnostic (Sec.~\ref{sec:tunneling_experiments}, Fig.~\ref{fig:tsne_tunneling}).}
Along a recorded ADIP trajectory, from each point
(and from SCIP incumbent snapshots), we run projected gradient ascent of the continuous
objective subject to the knapsack boundary
($w^\top \tilde x \leq C$), with step size $0.05$, at most $3000$ iterations, and
tolerances $10^{-6}$ (capacity) / $10^{-8}$ (projected gradient).
Two endpoints are counted as distinct if their $\ell_2$ distance exceeds $10^{-3}$.
Table~\ref{tab:basin_analysis} reports trajectory length, the number of distinct endpoints,
and the objective at the incumbent-started vs.\ best endpoint.
The 3D visualization in Fig.~\ref{fig:tsne_tunneling} is a t-SNE projection of the same
points (perplexity $30$) with objective on the vertical axis; basin identity is
always decided in the original $[0,1]^n$ space, not in the t-SNE embedding.

\paragraph{Software.}
Experiments used Python~$\ge 3.10$, NumPy~$<2$, PyTorch~$\ge 2$,
\texttt{pyscipopt} / SCIP~10, HiGHS (\texttt{highspy}), and OR-Tools (CBC + CP-SAT).
Pinned source hashes and reproduction commands are included in the repository.

\paragraph{Additional experimental figure.}
\label{app:gap-scatter}

Figure~\ref{fig:gap-scatter} plots, for the SCIP $n=400$ matched-budget sweep
(100 seeds), the percentage improvement of SCIP+ADIP over the matched SCIP
objective against the SCIP MIP gap at $T_{\mathrm{seed}}$.
Larger residual gaps after the seed phase tend to leave more room for ADIP
improvement, consistent with Sec.~\ref{sec:main-exp}, but the gap is not a
deterministic predictor of ADIP success.

\begin{figure}[t]
\centering
\includegraphics[width=0.85\linewidth]{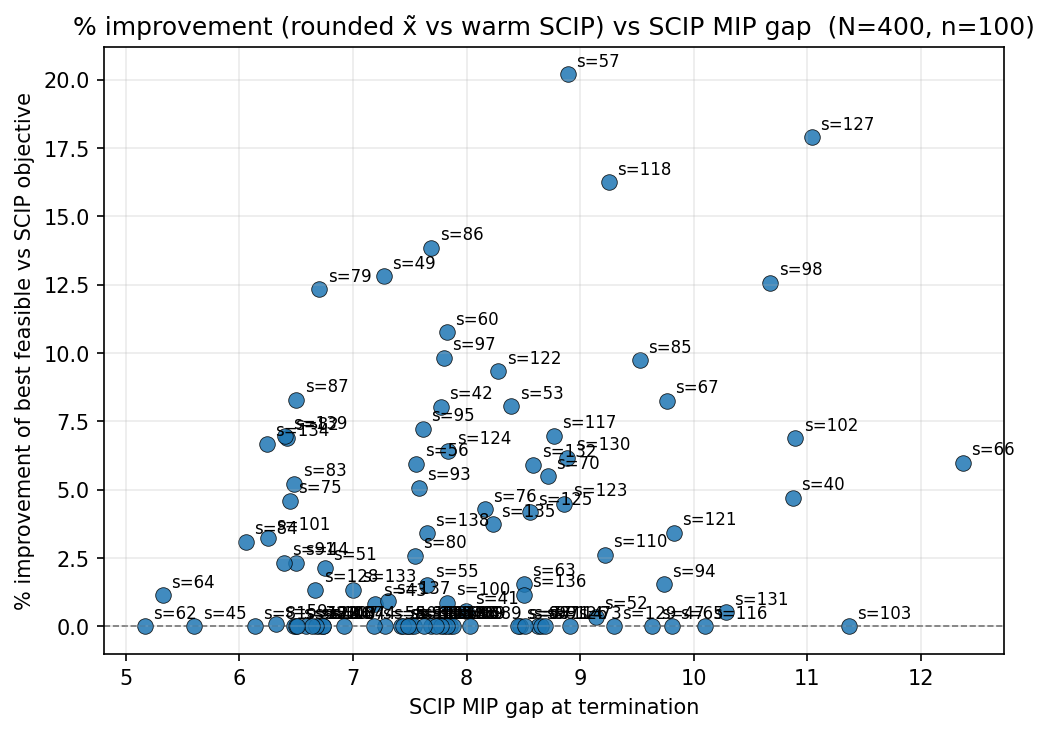}
\caption{Percentage improvement of SCIP+ADIP vs.\ matched SCIP, plotted against
the SCIP MIP gap after $T_{\mathrm{seed}}=500$\,s ($n=400$, 100 seeds).}
\label{fig:gap-scatter}
\end{figure}

\subsection{Proof of Theorem \ref{thm:incumbent-hardness}}\label{incumbent-hardness}

Membership in NP is immediate: given a candidate $x\in{0,1}^n$, one can verify both feasibility and strict improvement in polynomial time.

For NP-hardness, consider the decision version of QKP. Given $Q\in\mathbb{Z}^{n\times n}$, strictly positive weights $w\in\mathbb{Z}_{>0}^n$, capacity $C\in\mathbb{Z}_{>0}$, and an integer threshold $K\geq 1$, the problem asks whether there exists $x\in{0,1}^n$ such that
\begin{equation}
w^\top x \leq C,
\qquad
x^\top Qx \geq K.
\end{equation}
This decision problem is NP-hard \citep{gallo1980quadratic,pisinger2007solution}.

Construct an augmented instance by introducing one additional binary variable $z$. Give this variable weight $C$ and define
\begin{equation}
Q'
=
\begin{pmatrix}
Q & 0 \\
0 & K-1
\end{pmatrix},
\qquad
w'
=
\begin{pmatrix}
w \
C
\end{pmatrix}.
\end{equation}
The capacity remains $C$. Consider the explicitly known incumbent
\begin{equation}
x^{\inc}
=
(0,\ldots,0,1).
\end{equation}
It is feasible and has objective value $K-1$.

Because all original weights are strictly positive, any feasible solution with $x_{n+1}=1$ must set every original variable to zero: the additional variable alone exhausts the entire capacity. Hence no feasible solution with $x_{n+1}=1$ can strictly improve upon $x^{\inc}$. A strictly improving solution must therefore have $x_{n+1}=0$, in which case its objective is precisely $x^\top Qx$. Since the objective is integer-valued,
\begin{equation}
x^\top Qx > K-1
\quad\Longleftrightarrow\quad
x^\top Qx \geq K.
\end{equation}
Thus the augmented instance admits a feasible solution strictly improving upon $x^{\inc}$ if and only if the original quadratic knapsack decision instance is a YES instance. The construction is polynomial, which proves NP-hardness.

\subsection{Proof of Proposition \ref{prop:sequence-surjectivity}}\label{sequence-surjectivity}

Fix an arbitrary $x^\star\in\{0,1\}^n$. Consider the constant sequence-to-sequence map
\begin{equation}
f_{x^\star}(z)=x^\star
\end{equation}
on any compact input domain $\mathcal{K}$. This map is continuous. By the universal approximation theorem for transformers with positional encodings \citep{yun2020transformers}, for every $\varepsilon>0$ there exists a transformer $T_\theta\in\mathcal{T}$ such that
\begin{equation}
\sup_{z\in\mathcal{K}}
\|T_\theta(z)-f_{x^\star}(z)\|_\infty
<
\varepsilon.
\end{equation}
In particular, there is an input initiating the autoregressive rollout yielding,
\begin{equation}
\|\tilde{x}_\theta-x^\star\|_\infty
<
\varepsilon.
\end{equation}

Choose $\varepsilon<1/2$. If $x_i^\star=0$, then
\begin{equation}
\tilde{x}_{\theta,i}<\varepsilon<1/2,
\end{equation}
whereas if $x_i^\star=1$, then
\begin{equation}
\tilde{x}_{\theta,i}>1-\varepsilon>1/2.
\end{equation}
Thresholding each coordinate at $1/2$ therefore yields
\begin{equation}
\hat{x}_{\theta,i}=x_i^\star
\end{equation}
for every $i$. Since $x^\star$ was arbitrary, the hard-decoded transformer parametrization covers the entire binary hypercube $\{0,1\}^n$.

\subsection{Proof of Corollary \ref{cor:theta-hardness}}\label{theta-hardness}

By Proposition~\ref{prop:sequence-surjectivity},
\begin{equation}
\{\hat{x}_{\theta}:\theta\in\Theta\}
=
\{0,1\}^n.
\end{equation}
Therefore optimization over the hard-decoded transformer contains the original binary quadratic program exactly:
\begin{equation}
\max_{\theta:,A\hat{x}_{\theta}\leq b}
\hat{x}_{\theta}^{\top}Q\hat{x}_{\theta}
=
\max_{x\in{0,1}^n:,Ax\leq b}
x^\top Qx.
\end{equation}
Since the decision version of the problem on the right-hand side is NP-hard, so is global optimization over $\theta$.

\subsection{Proof of Proposition \ref{prop:lambda-bound}}\label{lambda-bound}

Because $A$, $b$, and $x$ are integer-valued, $Ax-b$ is integer-valued. Hence if a binary vector $x$ is infeasible, at least one constraint is violated by at least one unit, and therefore
\begin{equation}
\Phi(x)\geq1.
\end{equation}

Furthermore, since $x_i x_j\in{0,1}$,
\begin{equation}
x^\top Qx
=
\sum_{i,j}Q_{ij}x_ix_j
\leq
\sum_{i,j}\max{Q_{ij},0}
=
U_Q.
\end{equation}
Thus every infeasible binary $x$ satisfies
\begin{align}
x^\top Qx-\lambda\Phi(x)
&\leq
U_Q-\lambda <
f_{\inc},
\end{align}
where the final inequality follows from the condition in Proposition~\ref{prop:lambda-bound}.

Since $x^{\inc}$ is feasible, $\Phi(x^{\inc})=0$, and its penalized objective is exactly $f_{\inc}$. Therefore no infeasible binary point can globally maximize the penalized objective.

\subsection{Proof of Proposition \ref{prop:gumbel-gradient}}\label{gumbel-gradient}

The almost-sure convergence follows directly from the zero-temperature limit of the sigmoid:
\begin{equation}
\sigma\left(\frac{v}{\tau}\right)
\longrightarrow
\mathbf{1}_{v>0}
\end{equation}
for every $v\neq0$. 

Let $s'(u)=\sigma(u)(1-\sigma(u))$. Since $V=\Delta+G$ has density $q(v-\Delta)$,
\begin{align}
\mathbb{E}[D_\tau]
&=
\int_{\mathbb{R}}
\frac{1}{\tau}
s'\left(\frac{v}{\tau}\right)
q(v-\Delta),dv.
\end{align}
Using the substitution $v=\tau u$ gives
\begin{equation}
\mathbb{E}[D_\tau]
=
\int_{\mathbb{R}}
s'(u)
q(\tau u-\Delta),du.
\end{equation}
By dominated convergence and $\int_{\mathbb{R}}s'(u),du=1$,
\begin{equation}
\mathbb{E}[D_\tau]
\longrightarrow
q(-\Delta).
\end{equation}

Similarly,
\begin{align}
\mathbb{E}[D_\tau^2]
&=
\int_{\mathbb{R}}
\frac{1}{\tau^2}
s'\left(\frac{v}{\tau}\right)^2
q(v-\Delta),dv \
&=
\frac{1}{\tau}
\int_{\mathbb{R}}
s'(u)^2
q(\tau u-\Delta),du.
\end{align}
Since
\begin{equation}
\int_{\mathbb{R}}s'(u)^2,du
=
\int_0^1 t(1-t),dt
=
\frac{1}{6},
\end{equation}
we obtain
\begin{equation}
\mathbb{E}[D_\tau^2]
=
\frac{q(-\Delta)}{6\tau}
+
o(\tau^{-1}).
\end{equation}
Because $\mathbb{E}[D_\tau]$ converges to the finite value $q(-\Delta)$, the second moment dominates asymptotically, yielding
$\operatorname{Var}(D_\tau)=\Theta(\tau^{-1})$.

\subsection{Proof of Theorem \ref{thm:path-dependent}}\label{path-dependent}

We start by stating a formal version of the theorem.

Consider one attention head with input representation $H$, query, key, and value matrices
\begin{equation}
Q = H W_Q,\qquad
K = H W_K,\qquad
V = H W_V,
\end{equation}
and causally masked attention
\begin{equation}
\operatorname{Attn}(H)
=
\operatorname{softmax}
\left(
\frac{QK^\top}{\sqrt{d_k}}+M
\right)V,
\label{eq:causal-attention}
\end{equation}
where $M$ is the causal mask.

\begin{theorem}[Path-dependent induced geometry]
\label{thm:path-dependent}
Let $\theta$ be the parameters of a causal transformer and suppose that, for some attention head, the output of the relaxed autoregressive rollout depends nontrivially on at least one of its query or key matrices, i.e.
\begin{equation}
J_Q
:=
\frac{\partial \tilde{x}_\theta}{\partial W_Q}
\neq 0
\qquad\text{or}\qquad
J_K
:=
\frac{\partial \tilde{x}_\theta}{\partial W_K}
\neq 0.
\label{eq:attention-nondegeneracy}
\end{equation}
For every $c>0$, define a reparametrized transformer $\theta^{(c)}$ by
\begin{equation}
W_Q^{(c)}=cW_Q,
\qquad
W_K^{(c)}=c^{-1}W_K,
\label{eq:qk-rescaling}
\end{equation}
with every other parameter unchanged.

Then:

\begin{enumerate}
\item the two transformers define exactly the same autoregressive map,
\begin{equation}
\tilde{x}_{\theta^{(c)}}(z)
=
\tilde{x}_{\theta}(z)
\qquad
\text{for every input or autoregressive prefix }z;
\label{eq:same-transformer-function}
\end{equation}

\item nevertheless, except in the degenerate case $J_Q=J_K=0$, there exists $c\neq1$ for which
\begin{equation}
    J_{\theta^{(c)}}J_{\theta^{(c)}}^\top
    \neq
    J_{\theta}J_{\theta}^\top.
    \label{eq:different-induced-metric}
\end{equation}

\item there exists a QKP objective for which the two transformers induce different instantaneous velocities in relaxed decision space
\begin{equation}
\dot{\tilde{x}}_{\theta} \neq
\dot{\tilde{x}}_{\theta^{(c)}}.
\end{equation}
\end{enumerate}

In other words there exist two transformers representing exactly the same binary solution but inducing different local optimization geometries in relaxed decision space.
\end{theorem}

\begin{proof}
Under the rescaling stated in Theorem~\ref{thm:path-dependent}, the query and key matrices become
\begin{equation}
Q^{(c)}
=
H(cW_Q)
=
cQ
\end{equation}
and
\begin{equation}
K^{(c)}
=
H(c^{-1}W_K)
=
c^{-1}K.
\end{equation}
Their product is therefore invariant:
\begin{equation}
Q^{(c)}(K^{(c)})^\top
=
(cQ)(c^{-1}K)^\top
=
QK^\top.
\end{equation}
Hence the matrix entering the causally masked attention softmax is unchanged. Since $W_V$ and every subsequent model parameter are also unchanged, the output of this attention head is identical. Applying the same argument throughout the autoregressive rollout shows that
\begin{equation}
\tilde{x}_{\theta^{(c)}}(z)
=
\tilde{x}_{\theta}(z)
\end{equation}
for every prefix $z$. This proves the first statement.

We now examine the parameter Jacobian. The attention scores depend on $W_Q$ and $W_K$ through
\begin{equation}
W_QW_K^\top.
\end{equation}
For a perturbation $\delta W_Q$ around the reparametrized point,
\begin{equation}
\delta
\left(
W_Q^{(c)}(W_K^{(c)})^\top
\right)
=
\delta W_Q
(c^{-1}W_K)^\top,
\end{equation}
so the sensitivity to $W_Q$ is multiplied by $c^{-1}$. Similarly, the sensitivity to $W_K$ is multiplied by $c$. Consequently,
\begin{equation}
J_Q^{(c)}
=
c^{-1}J_Q,
\qquad
J_K^{(c)}
=
cJ_K.
\label{eq:jacobian-rescaling}
\end{equation}

Let $J_{\mathrm{rest}}$ denote the Jacobian with respect to all remaining parameters. Since the forward computation is unchanged under this rescaling, their contribution is unchanged. Hence
\begin{align}
J_{\theta^{(c)}}J_{\theta^{(c)}}^\top
&=
J_{\mathrm{rest}}J_{\mathrm{rest}}^\top
+
c^{-2}J_QJ_Q^\top
+
c^2J_KJ_K^\top.
\label{eq:metric-c}
\end{align}
At $c=1$,
\begin{equation}
J_{\theta}J_{\theta}^\top
=
J_{\mathrm{rest}}J_{\mathrm{rest}}^\top
+
J_QJ_Q^\top
+
J_KJ_K^\top.
\end{equation}

Suppose, for contradiction, that these two matrices were equal for every $c>0$. Then
\begin{equation}
(c^{-2}-1)J_QJ_Q^\top
+
(c^2-1)J_KJ_K^\top
=
0
\end{equation}
for every $c>0$. Taking $c\rightarrow\infty$ implies
\begin{equation}
J_KJ_K^\top=0,
\end{equation}
and therefore $J_K=0$. Taking $c\rightarrow0$ similarly yields $J_Q=0$. This contradicts the nondegeneracy assumption in Theorem~\ref{thm:path-dependent}.

Therefore, whenever the attention head affects the generated solution nontrivially, there exists $c\neq1$ such that
\begin{equation}\label{eq:jacobian-kernel-diff}
J_{\theta^{(c)}}J_{\theta^{(c)}}^\top
\neq
J_{\theta}J_{\theta}^\top.
\end{equation}

Combining Eq. \ref{eq:jacobian-kernel-diff} with Theorem~\ref{thm:induced-flow}, the two functionally equivalent initializations satisfy
\begin{align}
&\dot{\tilde{x}}_{\theta}
=
K_{\theta}
\nabla_{\tilde{x}}F_\lambda(\tilde{x}),\\&
\dot{\tilde{x}}^{(c)}_\theta
=
K_{\theta^{(c)}}
\nabla_{\tilde{x}}F_\lambda(\tilde{x}),
\end{align}
where
\begin{equation}
K_{\theta}
\neq
K_{\theta^{(c)}}
\end{equation}
for an appropriate $c$. Thus two transformers representing the same current solution can induce different trajectory kernels solely because of how that solution is encoded in weight space.

We next show that there exists a quadratic objective
\begin{equation}
f(\tilde{x})
=
\tilde{x}^\top Q\tilde{x}
\end{equation}
for which the two transformers induce different instantaneous trajectories:
\begin{equation}
K_{\theta}\nabla f(\tilde{x})
\neq
K_{\theta^{(c)}}\nabla f(\tilde{x}).
\end{equation}

Let
\begin{equation}
D
:=
K_{\theta}
-
K_{\theta^{(c)}}.
\end{equation}
Since $D\neq0$, there exists some $g\in\mathbb{R}^n$ such that
\begin{equation}
Dg\neq0.
\label{eq:distinguishing-gradient}
\end{equation}

We show that $g$ can be realized as the gradient of a quadratic objective at $\tilde{x}$. Let
\begin{equation}
r=\tilde{x}^\top\tilde{x}>0
\end{equation}
and define the symmetric matrix
\begin{equation}
S
=
\frac{
g\tilde{x}^\top+\tilde{x}g^\top
}{r}
-
\frac{
\tilde{x}^\top g
}{r^2}
\tilde{x}\tilde{x}^\top.
\label{eq:construct-S}
\end{equation}
A direct calculation gives
\begin{equation}
S\tilde{x}=g.
\end{equation}
Set
\begin{equation}
Q=\frac{1}{2}S.
\end{equation}
Because $S$ is symmetric,
\begin{equation}
\nabla_{\tilde{x}}
\left(
\tilde{x}^\top Q\tilde{x}
\right)
=
(Q+Q^\top)\tilde{x}
=
S\tilde{x}
=
g.
\end{equation}
Therefore
\begin{align}
\dot{\tilde{x}}_{\theta}
-
\dot{\tilde{x}}_{\theta^{(c)}}
=
\left(
K_{\theta}
-
K_{\theta^{(c)}}
\right)g=
Dg
\neq0.
\end{align}
Thus the different geometries produced by the self-attention symmetry can lead to genuinely different trajectories for an objective of exactly the quadratic form considered in this proof.

\end{proof}
\end{document}